\documentclass[11pt]{article}

\usepackage[utf8]{inputenc}
\usepackage[T1]{fontenc}
\usepackage{amsmath,amsfonts,amssymb,amsthm,mathtools}
\usepackage[authoryear,round]{natbib}
\usepackage{graphicx}
\usepackage{xcolor}
\usepackage{fullpage}
\usepackage{microtype}
\usepackage{csquotes}
\usepackage[osf,sc]{mathpazo}
\usepackage[scaled=0.90]{helvet}
\usepackage[scaled=0.85]{beramono}
\usepackage{textcomp}
\usepackage[small]{caption}
\usepackage[skip=2mm,indent=5mm]{parskip}
\usepackage{tikz}
\usetikzlibrary{shapes,arrows}

\definecolor{DarkGreen}{rgb}{0.075,0.375,0.075}
\definecolor{DarkRed}{rgb}{0.5,0.1,0.1}
\definecolor{DarkBlue}{rgb}{0.1,0.1,0.5}
\definecolor{Gray}{rgb}{0.2,0.2,0.2}

\usepackage{hyperref}
\hypersetup{
unicode=true,
pdftoolbar=true,
pdfmenubar=true,
pdffitwindow=false,
pdfnewwindow=true,
colorlinks=true,
linkcolor=DarkBlue,
citecolor=DarkGreen,
filecolor=DarkRed,
urlcolor=DarkBlue,
pdftitle={Retraining Seeks Stable Signals},
pdfauthor={Moritz Hardt},
}
\usepackage{cleveref}
\crefname{theorem}{Theorem}{Theorems}
\Crefname{theorem}{Theorem}{Theorems}
\crefname{lemma}{Lemma}{Lemmas}
\Crefname{lemma}{Lemma}{Lemmas}
\crefname{proposition}{Proposition}{Propositions}
\Crefname{proposition}{Proposition}{Propositions}
\crefname{corollary}{Corollary}{Corollaries}
\Crefname{corollary}{Corollary}{Corollaries}
\crefname{assumption}{Assumption}{Assumptions}
\Crefname{assumption}{Assumption}{Assumptions}

\theoremstyle{plain}
\newtheorem{theorem}{Theorem}

\newtheorem{proposition}[theorem]{Proposition}
\newtheorem{corollary}[theorem]{Corollary}
\newtheorem{claim}[theorem]{Claim}

\theoremstyle{definition}
\newtheorem{definition}[theorem]{Definition}

\newtheorem*{example}{Example}

\theoremstyle{remark}
\newtheorem{remark}[theorem]{Remark}

\newcommand{\R}{\mathbb{R}}
\newcommand{\cH}{\mathcal{H}}
\newcommand{\cX}{\mathcal{X}}

\newcommand{\cF}{\mathcal{F}}

\DeclareMathOperator*{\argmin}{arg\,min}

\DeclareMathOperator*{\E}{\mathbb{E}}

\DeclareMathOperator{\logit}{logit}
\DeclareMathOperator{\softmax}{softmax}

\title{Retraining Seeks Stable Signals}
\author{Moritz Hardt\footnote{Max Planck Institute for Intelligent Systems, T\"ubingen, and T\"ubingen AI Center. Email: \url{sf-admin@is.mpg.de}}}

\begin{document}
\maketitle

\begin{abstract}
Predictive models deployed at scale influence future data, a phenomenon called performativity. And there is always one way to cope: Train the model on new data, deploy it again, and repeat. This process, called retraining or repeated risk minimization, creates a feedback loop between model and data that real-world learning systems can’t avoid.
Results on performative prediction shed light on this dynamic: If the model's influence on the data is small, retraining reaches a fixed point. What remains open is why fixed points should naturally exist, and what governs retraining when the model's influence is strong. In this work we develop a new perspective on retraining---the stable signal principle---that addresses these questions. We start from the assumption that the prediction target has at least some small model-independent component, a stable signal, such as the intrinsic quality of an item.
We prove that when a nonzero stable signal exists, repeated risk minimization, suitably regularized, converges geometrically to the direction of this stable signal. This is true even if the model's influence on the target is arbitrarily large relative to the stable signal. Regularization emerges naturally as a force to control performativity, rather than to promote generalization, revealing a new facet of an old concept. We extend the analysis to a broad family of affine retraining operators under arbitrary model-induced feature changes, heterogeneous time-varying effects, and nonlinear responses. The stable signal perspective also applies to data feedback loops in language modeling, providing new explanations for the stability of language model training from model-generated data.

\end{abstract}

\section{Introduction}

Learning systems rarely operate in stationary environments for a fundamental reason: The model moves the data. This phenomenon, known as performativity, arises with all predictions that people react to: election polls, epidemiological forecasts, credit scores, online recommendations, digital ads, traffic estimates, chatbots, and AI assistants. In each case, people respond to a prediction in a way that may change the outcome, thus possibly reinforcing or invalidating the prediction. Online recommendations, for example, can be self-fulfilling prophecies, as users click on what the model ranks highly. Traffic predictions, on the other hand, can be self-negating: drivers jam up the recommended route. The strength and direction of performativity vary from one case to the other.

The way to cope in all applications is \emph{retraining}: periodically train the model on newly generated data, deploy it again, repeat indefinitely. Retraining creates a closed loop between model and data that drives learning systems into unknown behavior. Powerful results in performative prediction show that retraining on smooth, strongly convex losses converges to fixed points when the model's influence is small \citep{perdomo2020performative,mendler2020stochastic,drusvyatskiy2023stochastic}. But what happens when the model's influence is not small? In real settings, model influence isn't an approximation error; it's a first-order term \citep{mendler2024engine}. Some worry that the feedback loop will amplify the model’s behavior, causing harm across many areas of social life \citep{ensign2018runaway, jiang2019degenerate, ward2022loop, taori2023data}. Others hypothesize that the model quality degrades sharply and eventually collapses entirely \citep{shumailov2024collapse, alemohammad2024mad}. 

We show that a different principle is at play: Retraining seeks \emph{stable signals}---model-independent factors in outcomes---even if weak compared to the model’s own influence. Specifically, we prove several convergence results for retraining that all follow from the same organizing principle: The existence of a small stable signal overcomes strong performative effects under repeated risk minimization, when suitably normalized or regularized.

\begin{itemize}
\item To start, we illustrate the stable signal principle in a population least-squares model. Despite arbitrary model-dependent changes in the feature distribution, retraining follows an explicit affine recursion in prediction space. Suitable $\ell_2$-regularization yields geometric convergence to the stable-signal direction, while two consecutive updates identify the feedback strength and recover the performatively stable point of the unregularized problem. We analyze heterogeneous and time-varying performative effects, deriving explicit population recursions, convergence conditions, and limits on what can be identified from a retraining trajectory.

\item The phenomenon is not specific to squared loss: Whenever the loss-optimal conditional prediction is affine in the deployed predictor, population retraining inherits the same affine map. This generalization covers logistic responses and softmax models, as well as more general nonlinear response links.

\item As an alternative to $\ell_2$-regularization, we show that normalization provides a second way to control strong performative feedback. Using ideas from the convergence analysis of the power method, we prove that normalized retraining recovers the stable-signal direction under arbitrarily strong positive isotropic feedback. The angular convergence analysis extends to operator-valued and time-varying effects.

\item We apply the stable signal perspective to language-model data feedback loops. For repeated training on mixtures of real and model-generated data, we characterize when the influence of initialization disappears and how retaining past generated data changes the rate of convergence.

\item Finally, we give a parametric version of the affine theory and stochastic-gradient convergence guarantees for squared loss in the standard parametric formulation of performative prediction.
\end{itemize}

\subsection{Related work}
\label{sec:related-work}

Performative prediction studies learning problems in which model deployment changes the data distribution \citep{perdomo2020performative}. A predictor may be performatively stable, as a fixed point of retraining, or performatively optimal, as a minimizer of its own deployed risk. Related ideas go back almost a century to work on economic forecasting \citep{morgenstern1928wirtschaftsprognose,merton1948self,grunberg1954predictability,simon1954bandwagon}. More recently, special cases of performative prediction appeared in computer science in areas such as strategic classification \citep{bruckner2011stackelberg,hardt2016strategic}, algorithmic fairness \citep{ensign2018runaway,hashimoto2018fairness,liu2018delayed}, and concept/distribution shift \citep{gama2014survey}. The standard formulation of a distribution map allows deployment to change the feature distribution as well as the conditional outcome variable. The restriction to fixed feature distributions is outcome performativity \citep{kim2023making,perdomo2025revisiting}. See \citet{hardt2025performative} and \citet{kehrenberg2026dissecting} for surveys.

Standard convergence results require the distributional sensitivity to model parameters to be small relative to loss curvature \citep{perdomo2020performative,mofakhami2023performative}. Extensions cover stochastic algorithms, gradual response, games, and mixed sources of shift \citep{mendler2020stochastic,drusvyatskiy2023stochastic,brown2022performative,izzo2022how,li2022state,ray2022decision,narang2023multiplayer,lee2026partially,farina2026stability}. 
Work on performative optimality uses derivative-free, bandit, and plug-in methods \citep{izzo2021how,miller2021outside,jagadeesan2022regret,lin2024plugin}. Calibration gives another route to analyze outcome performativity~\citep{perdomo2025revisiting}.

Data feedback loops arise in predictive policing, recommendation, and generative modeling \citep{ensign2018runaway,jiang2019degenerate,taori2023data}. Repeated training on generated text can lose distributional support or otherwise degenerate \citep{shumailov2024collapse,alemohammad2024mad,dohmatob2024tale}, prompting concerns about ``model collapse''. Fresh or accumulated real data, however, can stabilize these loops \citep{bertrand2024stability,gerstgrasser2024model}. Our results on language model retraining echo these stability results and relate them to a broader phenomenon.

\section{Preliminaries}

Let $D_0$ be a fixed reference distribution on a bounded feature domain $\cX$. We use it to compare prediction functions and work in the Hilbert space $\cH=L^2(D_0)$, with inner product
\[
\langle f,g\rangle = \E_{D_0}[f(X)g(X)]
\]
and norm $\|f\|=\sqrt{\langle f,f\rangle}$. Unless otherwise specified, every norm of a prediction function refers to this norm. Equalities between functions hold $D_0$-almost everywhere.

\paragraph{Distribution map.}
Given a deployed predictor $h\in\cH$, let $D(h)$ denote the induced joint distribution of $(X,Y)$ and $D_X(h)$ its feature marginal. We call $D(\cdot)$ the distribution map. We place no smoothness or parametric restriction on the marginal map $h\mapsto D_X(h)$: deployment may reweight the feature population arbitrarily. We assume only that $D_X(h)$ and the reference distribution $D_0$ share the same null sets (the same support in finite spaces). 

\paragraph{Repeated risk minimization.}
Let $\ell$ be a loss function. Repeated risk minimization, also called retraining, is the update rule
\begin{equation*}
f_{t+1}\in \argmin_{f\in\cH}\; \E_{D(h_t)}\ell(f(X),Y)\,,
\end{equation*}
where $h_t$ is the predictor deployed at round $t$. In \Cref{sec:least-squares}, the latest iterate is deployed directly, $h_t=f_t$; in \Cref{sec:deployment-normalization}, the deployed predictor is the normalized iterate.

\begin{definition}
A deployed predictor $h$ is performatively stable for the loss $\ell$ if it is a fixed point of population retraining:
\[
h\in \argmin_{f\in\cH}\;\E_{D(h)}\ell(f(X),Y)\,.
\]
\end{definition}

\section{Repeated least squares converges to stable signals}
\label{sec:least-squares}

Assume that deployment of $h$ induces the conditional mean
\begin{equation*}
\E_{D(h)}[Y\mid X]=\alpha f^*(X)+\beta h(X)\,,
\end{equation*}
where $f^*$ is a stable signal, $\alpha>0$ is the stable-signal strength and $\beta\in\R$ is the performative feedback coefficient. Independently of this conditional-mean model, deployment may change the feature marginal $D_X(h)$ in any way that preserves support. We first take $\alpha$ and $\beta$ to be constant. \Cref{subsec:heterogeneous-effects} allows for heterogeneous performative effects.

\begin{example}[Watch-time prediction]
The outcome variable $Y\in\R$ represents the time a visitor spends watching a video~$X\in\R^d$. The stable signal $f^*(X)$ might represent a latent measure of \emph{video quality}. The prediction $f(X)\in\R$ represents the platform's predicted watch time. It influences the actual watch time, since the platform will display videos with high predicted watch time more prominently to the visitor. 
\end{example}

Our goal is to analyze repeated squared loss minimization in the above setting. Starting from a fixed predictor $f_0$, we repeatedly solve the regularized least squares objective:
\begin{equation*}
 f_{t+1}
=\argmin_{f\in\cH}\; \E_{D(f_t)}\left[(f(X)-Y)^2+\delta f(X)^2\right]\,,
\end{equation*}
for a penalty $\delta\ge0$ that we may choose. 

Although deployment may change the feature marginal arbitrarily, this does not affect the optimal solution to the problem.
\begin{claim}[Update rule]
\label{claim:ls-update-rule}
For every $t\ge0,$
\begin{equation}
\label{eq:ridge-update}
f_{t+1}=\frac{\alpha}{1+\delta}f^*+\frac{\beta}{1+\delta}f_t\,.
\end{equation}
\end{claim}
\begin{proof}
Let $w_t=dD_X(f_t)/dD_0$. By common support, $w_t(x)>0$ almost everywhere, and the objective equals
\[
   \int w_t(x)\left(
  \E_{D(f_t)}[(f(x)-Y)^2\mid X=x]+\delta f(x)^2
\right)dD_0(x)\,.
\]
The positive weight $w_t(x)$ changes how often each feature value appears, but not the pointwise minimizer. 
Hence, for every deployment-dependent marginal $D_X(f_t)$ with common support, we can compute the optimal solution pointwise by differentiation. Condition on $X=x$. Up to a term independent of the candidate value $a=f(x)$, the pointwise objective is
\[
\bigl(a-\E_{D(f_t)}[Y\mid X=x]\bigr)^2+\delta a^2\,.
\]
Its minimizer is
\[
a=\frac{\E_{D(f_t)}[Y\mid X=x]}{1+\delta}
  =\frac{\alpha f^*(x)+\beta f_t(x)}{1+\delta}\,.
\]
\end{proof}

Given the explicit update rule, we can characterize the convergence of repeated least squares.

\begin{proposition}[Repeated least-squares dynamics]
\label{prop:least-squares}
Assume $\alpha>0$ and $\delta\ge0$. If $\beta\ne1+\delta$, the update \cref{eq:ridge-update} has the unique fixed point
\[
f_{\delta}=\frac{\alpha}{1+\delta-\beta}f^*\,.
\]
Moreover,
\[
f_t-f_{\delta}
=\left(\frac{\beta}{1+\delta}\right)^t
  \bigl(f_0-f_{\delta}\bigr)\,.
\]
Unless $f_0=f_{\delta}$ already, $f_t\to f_{\delta}$ geometrically in norm if and only if $|\beta|<1+\delta$. For nonnegative performativity $\beta\ge0$, it suffices to choose $\delta>\beta-1$.
\end{proposition}

\begin{proof}
By Claim~\ref{claim:ls-update-rule},
\[
f_{t+1}=\frac{\alpha}{1+\delta}f^*+\frac{\beta}{1+\delta}f_t\,.
\]

The fixed-point equation is
\[
(1+\delta)f=\alpha f^*+\beta f\,,
\]
which gives the stated unique fixed point when $\beta\ne1+\delta$. Subtracting the fixed-point equation from \cref{eq:ridge-update} gives
\[
f_{t+1}-f_{\delta} =\frac{\beta}{1+\delta}\bigl(f_t-f_{\delta}\bigr)\,,
\]
and hence the expression.
\end{proof}

\paragraph{The role of regularization.} In learning theory, regularization typically promotes generalization. Here, it serves a different purpose. By increasing $\delta\ge0,$ we can control arbitrarily large performative effects $\beta>0.$ We only need to set $\delta=2\beta,$ for example, to guarantee geometric convergence of retraining to a fixed point. Regularization makes retraining contractive when it isn't on its own.

\subsection{Two-step decoding of performative stable points}

The fixed point of the regularized objective is not the fixed point of the unregularized least squares objective. The unregularized retraining map is
\[
T_0(f)=\alpha f^*+\beta f\,.
\]
If $\beta\ne1$, its unique performatively stable point is
\begin{equation*}
f_{\mathrm{PS}}=\frac{\alpha}{1-\beta}f^*\,.
\end{equation*}
The regularized fixed point $f_\delta$ is generally different, but in this affine model it contains the same one-dimensional signal. When $\beta$ is known, it can be decoded directly from $f_\delta$ via the identity
\begin{equation*}
f_{\mathrm{PS}}
=\frac{1+\delta-\beta}{1-\beta}f_\delta\,.
\end{equation*}
The point $f_{\mathrm{PS}}$ is performatively stable for the original squared loss. If $\beta=1$ and $\alpha>0$, no stable point exists. 

The previous identity gives us $f_{\mathrm{PS}}$ provided that we know $\beta$. Next we identify $\beta$ itself. 

\paragraph{Two-step decoding.}
The next proposition shows how to identify $\beta$, the strength of performativity, from three consecutive iterates, that is, two risk minimization steps. Once we know $\beta$ we also get the stable signal $\alpha f^*$ and the stable point $f_{\mathrm{PS}}$ for free.

\begin{proposition}[Identification from two steps]
\label{prop:identify-beta}
Suppose the population iterates satisfy \cref{eq:ridge-update}, and observe $f_t,f_{t+1},f_{t+2}$ for some $t$ with $f_{t+1}\ne f_t$. Define
\begin{equation}
\label{eq:identify-effective-beta}
b_t = \frac{\langle f_{t+2}-f_{t+1},\,f_{t+1}-f_t\rangle}
 {\|f_{t+1}-f_t\|^2}\,.
\end{equation}
Then,
\[
\beta=(1+\delta)b_t,\quad\text{and}\quad 
\alpha f^*=(1+\delta)\bigl(f_{t+1}-b_t f_t\bigr)\,.
\]
Consequently, if $\beta\ne1$, the performatively stable point for the original squared loss objective 
is
\begin{equation}
\label{eq:identify-original-stable-point}
f_{\mathrm{PS}}
= \frac{(1+\delta)\bigl(f_{t+1}-b_t f_t\bigr)}{1-(1+\delta)b_t}\,.
\end{equation}
If $\|f^*\|=1$ and $\alpha>0$, then $\alpha=\|\alpha f^*\|$ and $f^*=(\alpha f^*)/\alpha$.
\end{proposition}

\begin{proof}
Subtract \cref{eq:ridge-update} at rounds $t+1$ and $t$ to obtain
\[
f_{t+2}-f_{t+1} =\frac{\beta}{1+\delta} \bigl(f_{t+1}-f_t\bigr)\,.
\]
Because $f_{t+1}\ne f_t$, taking the inner product with $f_{t+1}-f_t$ gives \cref{eq:identify-effective-beta}. Substituting the recovered coefficient into \cref{eq:ridge-update} yields the stable component $\alpha f^*$, and division by $1-\beta$ gives \cref{eq:identify-original-stable-point}.
\end{proof}

Identification does not require convergence. It applies to any two consecutive updates for which the iterate moves. We can therefore cut retraining short after two iterations and directly jump to the stable point. 

\subsection{Heterogeneous performative effects}
\label{subsec:heterogeneous-effects}

The same convergence analysis allows for feature-dependent coefficients:
\[
\E_{D(h)}[Y\mid X=x] =\alpha(x)f^*(x)+\beta(x)h(x)\,.
\]
We can only identify $\alpha(x)f^*(x)$ without an additional normalization, so write
\[
s(x)=\alpha(x)f^*(x)
\]
for the stable component of the conditional mean. Assume $s\in\cH$ and that $\beta$ is measurable and essentially bounded. The same pointwise minimization as above gives
\[
f_{t+1}(x) =\frac{s(x)+\beta(x)f_t(x)}{1+\delta}\,.
\]
Whenever $1+\delta-\beta(x)\neq 0$ almost everywhere, the regularized retraining map has the pointwise fixed point
\[
f_\delta(x) =\frac{s(x)}{1+\delta-\beta(x)}\,.
\]
Moreover,
\begin{equation*}
f_t(x)-f_\delta(x)
=\left(\frac{\beta(x)}{1+\delta}\right)^t
\bigl(f_0(x)-f_\delta(x)\bigr)\,.
\end{equation*}
In particular, the uniform condition $\|\beta\|_\infty<1+\delta$ implies geometric convergence in norm at rate at most $\|\beta\|_\infty/(1+\delta)$.

For the original squared-loss objective, the performatively stable point is
\[
f_{\mathrm{PS}}(x) =\frac{s(x)}{1-\beta(x)}\,,
\]
provided $1-\beta(x)\neq 0$ almost everywhere and the right-hand side belongs to
$\cH$. The stable point is now a pointwise, not global, rescaling of the stable
component. It need not be collinear with $f^*$, and its sign may differ across
feature values.

We can still decode all relevant parameters pointwise. Letting $\Delta_t(x)=f_{t+1}(x)-f_t(x)$, the performative feedback function comes out of
\begin{equation*}
\Delta_{t+1}(x) =\frac{\beta(x)}{1+\delta}\Delta_t(x)\,,
\qquad \beta(x) =(1+\delta)\frac{\Delta_{t+1}(x)}{\Delta_t(x)}
\end{equation*}
for every $x$ such that $\Delta_t(x)\neq 0$. 

\subsection{Time-varying heterogeneous effects}
\label{subsec:time-varying-heterogeneous-effects}

The distribution map may also change with time, because the environment drifts for reasons beyond the deployed predictor or because its response evolves across rounds. This possibility is central in stateful and partially performative models \citep{brown2022performative,lee2026partially}. Let~ $D_t(h)$ denote the joint distribution generated at round $t$ by deploying $h$. Suppose \begin{equation*}
\E_{D_t(f_t)}[Y\mid X=x]
=s_t(x)+\beta_t(x)f_t(x)\,,
\end{equation*}
where the feature marginal $D_{t,X}(f_t)$ may otherwise vary arbitrarily with $t$ and $f_t$, subject to common support with $D_0$. The same pointwise calculation gives
\begin{equation*}
f_{t+1}(x)
=\frac{s_t(x)+\beta_t(x)f_t(x)}{1+\delta}\,.
\end{equation*}

\begin{proposition}[Time-varying heterogeneous dynamics]
\label{prop:time-varying-heterogeneous-dynamics}
Assume $s_t\in\cH$ and $\beta_t\in L^\infty(D_0)$ for every $t$, and suppose
\begin{equation}
\label{eq:time-varying-uniform-contraction}
q:=\sup_{t\ge0}\frac{\|\beta_t\|_\infty}{1+\delta}<1\,.
\end{equation}
Then the effect of initialization decays geometrically: for any two initializations $f_0$ and $\widetilde f_0$, the resulting iterates satisfy
\begin{equation}
\label{eq:time-varying-forgetting}
\|f_t-\widetilde f_t\|
\le q^t\|f_0-\widetilde f_0\|\,.
\end{equation}
If $s_t\to s_\infty$ in norm and $\beta_t\to\beta_\infty$ in $L^\infty(D_0)$, then
\begin{equation}
\label{eq:time-varying-limit}
f_t\longrightarrow
\frac{s_\infty}{1+\delta-\beta_\infty}
\qquad\text{in norm}\,.
\end{equation}
\end{proposition}

\begin{proof}
Set $a_t=\beta_t/(1+\delta)$ and $d_t=s_t/(1+\delta)$. Subtracting two copies of the update gives
\[
f_{t+1}-\widetilde f_{t+1}
=a_t(f_t-\widetilde f_t)\,,
\]
so \cref{eq:time-varying-forgetting} follows from $\|a_t\|_\infty\le q$.

For the convergence statement, let
\[
f_\infty=\frac{s_\infty}{1+\delta-\beta_\infty}\,.
\]
The denominator is bounded away from zero by \cref{eq:time-varying-uniform-contraction}. Writing $e_t=f_t-f_\infty$, we have
\[
e_{t+1}=a_te_t+\varepsilon_t,
\qquad
\varepsilon_t=d_t+a_tf_\infty-f_\infty\,.
\]
The assumptions imply $\|\varepsilon_t\|\to0$. For any $T<t$,
\[
\|e_t\|
\le q^{t-T}\|e_T\|
  +\sum_{k=T}^{t-1}q^{t-1-k}\|\varepsilon_k\|\,.
\]
Taking the limit superior and then sending $T\to\infty$ proves \cref{eq:time-varying-limit}.
\end{proof}

Uniform contraction therefore guarantees that retraining forgets its initialization even if the environment does not settle. If $(s_t,\beta_t)$ is periodic, the update over one period is a contraction and there is a unique attracting periodic trajectory. Under arbitrary variation, there need not be a single stable point.

Time variation also changes what can be identified from one trajectory. Let
\[
a_t=\frac{\beta_t}{1+\delta},
\qquad
d_t=\frac{s_t}{1+\delta},
\qquad
\Delta_t=f_{t+1}-f_t\,.
\]
Then
\[
\Delta_{t+1} =a_{t+1}\Delta_t +(a_{t+1}-a_t)f_t +(d_{t+1}-d_t)\,.
\]
Thus consecutive increment ratios no longer isolate $\beta_t$ without additional restrictions on how $s_t$ and $\beta_t$ evolve.

\section{Affine retraining maps}
\label{sec:affine-retraining}

The least-squares calculation above is an instance of a more general fact. What matters for population retraining is the map from the deployed predictor to the loss-optimal prediction.

Let $\cF\subseteq\cH$ be a closed linear prediction space. At round $t$, deploying $h$ induces a distribution $D_t(h)$. Define the time-dependent conditional risk
\[
r_{t,h}(a,x)=\E_{D_t(h)}[\ell(a,Y)\mid X=x],
\qquad a\in\R\,.
\]

\begin{proposition}[Time-dependent affine population retraining]
\label{prop:affine-population-retraining}
Suppose there are $s_t\in\cF$ and bounded linear operators $B_t:\cF\to\cF$ such that, for every round $t$, every deployed $h\in\cF$, and almost every $x$,
\begin{equation}
\label{eq:affine-pointwise-minimizer}
\argmin_{a\in\R} r_{t,h}(a,x)
=s_t(x)+(B_th)(x)\,.
\end{equation}
Assume $D_{t,X}(h)$ and $D_0$ have the same null sets. Then population retraining satisfies
\begin{equation}
\label{eq:affine-retraining-map}
T_t(h)=s_t+B_th\,.
\end{equation}
If
\[
\sup_{t\ge0}\|B_t\|_{\mathrm{op}}\le q<1\,,
\]
then the effect of initialization decays geometrically: for any two initializations $h_0$ and $\widetilde h_0$, the resulting iterates satisfy
\begin{equation}
\label{eq:affine-time-varying-forgetting}
  \|h_t-\widetilde h_t\|
\le q^t\|h_0-\widetilde h_0\|\,.
\end{equation}
If in addition $s_t\to s_\infty$ in norm and $B_t\to B_\infty$ in operator norm, then
\begin{equation}
\label{eq:affine-time-varying-limit}
   h_t\longrightarrow (I-B_\infty)^{-1}s_\infty\,.
\end{equation}

In the time-invariant case $s_t=s$ and $B_t=B$, if $I-B$ is invertible, the unique performatively stable point is
\begin{equation}
\label{eq:affine-stable-point}
h_{\mathrm{PS}}=(I-B)^{-1}s\,,
\end{equation}
and
\begin{equation}
\label{eq:affine-error-recurrence}
h_t-h_{\mathrm{PS}}=B^t(h_0-h_{\mathrm{PS}})\,.
\end{equation}
Thus $\|B\|_{\mathrm{op}}<1$ implies geometric convergence from every initialization; in finite dimensions, convergence from every initialization is equivalent to $\rho(B)<1$.
\end{proposition}
\begin{proof}
By \cref{eq:affine-pointwise-minimizer},
\[
r_{t,h}(f(x),x)\ge r_{t,h}(s_t(x)+(B_th)(x),x)
\]
for every $f\in\cF$ and almost every $x$. Integrating under $D_{t,X}(h)$ shows that $s_t+B_th$ is the unique population minimizer.

Subtracting two copies of \cref{eq:affine-retraining-map} gives
\[
h_{t+1}-\widetilde h_{t+1}
=B_t(h_t-\widetilde h_t)\,,
\]
which proves \cref{eq:affine-time-varying-forgetting}. For the convergence statement, let
\[
h_\infty=(I-B_\infty)^{-1}s_\infty\,.
\]
This is well defined because $\|B_\infty\|_{\mathrm{op}}\le q<1$. With $e_t=h_t-h_\infty$,
\[
e_{t+1}
=B_te_t+(s_t-s_\infty)+(B_t-B_\infty)h_\infty\,.
\]
The last two terms tend to zero. Iterating the norm inequality and using $\|B_t\|_{\mathrm{op}}\le q$ proves \cref{eq:affine-time-varying-limit}. The time-homogeneous formulas follow by taking $s_t=s$ and $B_t=B$.
\end{proof}

\begin{remark}[Relation to heterogeneous effects]
Proposition~\ref{prop:time-varying-heterogeneous-dynamics} is the special case
\[
s_t\leftarrow \frac{s_t}{1+\delta},
\qquad
(B_th)(x)=\frac{\beta_t(x)}{1+\delta}h(x)\,.
\]
Here $B_t$ is a multiplication operator, so each feature value evolves independently. Proposition~\ref{prop:affine-population-retraining} is strictly more general because $B_t$ may couple predictions at different feature values or different coordinates.
\end{remark}

\subsection{Nonlinear performative effects}
\label{subsec:nonlinear-performative-effects}

An affine risk minimizer doesn't necessarily require linear performative effects. The next corollaries show that nonlinear performative responses can still lead to affine population dynamics for suitably chosen loss functions.

\begin{corollary}[Binary cross entropy with affine logits]
\label{cor:affine-logit-retraining}
Let $Y\in\{0,1\}$, let $\sigma(a)=(1+e^{-a})^{-1}$, and use logistic cross entropy
\[
\ell_{\log}(a,y)=\log(1+e^a)-ya\,.
\]
Suppose
\[
\Pr_{D(h)}\{Y=1\mid X=x\} =\sigma\bigl(s(x)+(Bh)(x)\bigr)\,.
\]
Then population retraining satisfies $T(h)=s+Bh$. Consequently, whenever $I-B$ is invertible,
\[
h_{\mathrm{PS}}=(I-B)^{-1}s, \qquad
h_t-h_{\mathrm{PS}}=B^t(h_0-h_{\mathrm{PS}})\,.
\]
\end{corollary}

\begin{proof}
Writing $p_h(x)=\Pr_{D(h)}\{Y=1\mid X=x\}$, the derivative of the conditional risk is
\[
\partial_a r_h(a,x)=\sigma(a)-p_h(x)\,.
\]
The conditional risk is strictly convex, so its unique minimizer is
$\logit p_h(x)=s(x)+(Bh)(x)$. Apply Proposition~\ref{prop:affine-population-retraining}.
\end{proof}

Thus the observed response can be nonlinear even in the simplest scalar model. For example,
\[
\Pr_{D(h)}\{Y=1\mid X=x\}
=\sigma\bigl(\alpha f^*(x)+\beta h(x)\bigr)
\]
is a nonlinear sigmoid response, while repeated logistic retraining follows the affine recursion
$T(h)=\alpha f^*+\beta h$.

The same pointwise argument applies to a vector of class logits.

\begin{corollary}[Multiclass cross entropy with affine logits]
\label{cor:affine-softmax-retraining}
Let $Y\in\{1,\ldots,K\}$ and let
\[
\R_0^K=\{a\in\R^K:\mathbf 1^\top a=0\}\,.
\]
Let $\cF\subseteq L^2(D_0;\R_0^K)$ be a closed linear space of centered-logit functions, equipped with the inner product
\[
\langle f,g\rangle=\E_{D_0}[f(X)^\top g(X)]\,.
\]
Use multiclass cross entropy,
\[
\ell_{\mathrm{ce}}(a,y) =-a_y+\log\sum_{j=1}^K e^{a_j}\,.
\]
If
\begin{equation}
\label{eq:affine-softmax-response}
 \Pr_{D(h)}\{Y=\cdot\mid X=x\}
=\softmax\bigl(s(x)+(Bh)(x)\bigr)\,,
\end{equation}
then population retraining in centered-logit space satisfies
\[
T(h)=s+Bh\,.
\]
Hence the stable point and convergence formulas in
\cref{eq:affine-stable-point,eq:affine-error-recurrence} hold without change.
\end{corollary}

\begin{proof}
For a conditional class distribution $p$, cross entropy at score $a$ equals
\[
H(p)+\mathrm{KL}\bigl(p\,\|\,\softmax(a)\bigr)\,.
\]
It is minimized when $\softmax(a)=p$. Restricting to $\R_0^K$ removes the additive-logit ambiguity, so the minimizer under \cref{eq:affine-softmax-response} is uniquely $s(x)+(Bh)(x)$. Apply Proposition~\ref{prop:affine-population-retraining}.
\end{proof}

The same conclusion holds for some nonlinear responses.

\begin{corollary}[Nonlinear performative responses]
\label{cor:nonlinear-response-link}
Let $\psi:\R\to(0,1)$ be injective and define the composite cross-entropy loss
\[
\ell_\psi(a,y)
=-y\log \psi(a)-(1-y)\log(1-\psi(a))\,.
\]
Suppose $Y\in\{0,1\}$ and
\begin{equation}
\label{eq:nonlinear-link-response}
\Pr_{D(h)}\{Y=1\mid X=x\}
=\psi\bigl(s(x)+(Bh)(x)\bigr)\,.
\end{equation}
Then population retraining satisfies $T(h)=s+Bh$, with stable point and error dynamics given by
\cref{eq:affine-stable-point,eq:affine-error-recurrence}.
\end{corollary}

\begin{proof}
For $p_h(x)=\psi(s(x)+(Bh)(x))$, the conditional risk is
\[
H(p_h(x)) +\mathrm{KL}\!\left(
  \operatorname{Ber}(p_h(x))
   \,\middle\|\,
   \operatorname{Ber}(\psi(a))
 \right)\,.
\]
It is minimized when $\psi(a)=p_h(x)$. Injectivity of $\psi$ gives the unique minimizer $a=s(x)+(Bh)(x)$, and the proposition applies.
\end{proof}

For example, consider
\[
\psi(a)=\sigma(a+\gamma\sin a), \qquad |\gamma|<1\,.
\]
This is an injective, nonlinear response link, so the probability in \cref{eq:nonlinear-link-response} is a nonlinear function of the deployed predictor. Nevertheless, retraining with the corresponding composite cross entropy follows the same affine recursion. 

\section{Normalization handles arbitrary positive feedback}
\label{sec:deployment-normalization}

Regularization stabilizes retraining by controlling the scale of the predictor. When only its direction matters at deployment (e.g., for ranking content), we can instead normalize after each retraining step. This preserves the direction of the update and, as we show below, can recover the stable signal even under arbitrarily strong positive feedback.

At round $t$, we deploy
\[
z_t=\frac{f_t}{\|f_t\|}, \qquad \|z_t\|=1\,.
\]
Assume the conditional response depends on this deployed predictor:
\[
   \E_{D(z_t)}[Y\mid X]=\alpha f^*(X)+\beta z_t(X)\,.
\]
The deployed feature distribution may also change with $z_t$. By the same pointwise calculation as in \Cref{sec:least-squares}, population squared-loss retraining gives
\[
f_{t+1}=\alpha f^*+\beta z_t\,,
\]
which is then normalized again before the next deployment:
\begin{equation}
\label{eq:normalized-update}
z_{t+1}=\frac{f_{t+1}}{\|f_{t+1}\|}
=\frac{\alpha f^*+\beta z_t}{\|\alpha f^*+\beta z_t\|}\,.
\end{equation}
The next theorem gives a global angle bound for every $\alpha>0$ and $\beta\ge0$.
\begin{theorem}[Normalization stabilizes arbitrarily strong positive feedback]
\label{thm:normalized-deployment}
Let $\alpha>0$, $\beta\ge0$, $\|f^*\|=1$, and $\|z_0\|=1$. Define the angle to the stable signal by
\[
\varphi_t=\arccos\langle z_t,f^*\rangle\in[0,\pi]\,.
\]
If $z_0\ne-f^*$, then the deployed predictors generated by \cref{eq:normalized-update} satisfy the global geometric bound
\[ 
\tan\!\left(\frac{\varphi_t}{2}\right)
  \le 
\left(\frac{\beta}{\alpha+\beta}\right)^t
\tan\!\left(\frac{\varphi_0}{2}\right)\,.
\]
In particular, $z_t\to f^*$ and $f_t\to (\alpha+\beta)f^*$. 
\end{theorem}

\begin{proof}
Let $u=f^*$ and write $z_t = \cos(\varphi_t)u+\sin(\varphi_t)v_t,$
where \(v_t\perp u\) and \(\|v_t\|=1\), and let
$r_t=\|\alpha u+\beta z_t\|.$ Then
\[
z_{t+1} = \frac{(\alpha+\beta\cos\varphi_t)u +\beta\sin\varphi_t\,v_t}{r_t}\,,
\]
so
\[
\cos\varphi_{t+1} = \frac{\alpha+\beta\cos\varphi_t}{r_t},
  \qquad
\sin\varphi_{t+1} = \frac{\beta\sin\varphi_t}{r_t}\,.
\]
Therefore,
\[
\tan\frac{\varphi_{t+1}}{2}
= \frac{\sin\varphi_{t+1}}{1+\cos\varphi_{t+1}}
 = \frac{\beta\sin\varphi_t}{r_t+\alpha+\beta\cos\varphi_t}\,.
\]
We now lower bound this denominator. By definition, 
\[
r_t^2=\alpha^2+\beta^2+2\alpha\beta\cos\varphi_t\,.
\]
Subtracting $(\beta+\alpha\cos\varphi_t)^2$ from both sides,
\begin{align*}
 r_t^2-(\beta+\alpha\cos\varphi_t)^2
&=
\alpha^2+\beta^2 +2\alpha\beta\cos\varphi_t
-(\beta+\alpha\cos\varphi_t)^2 \\
&= \alpha^2\sin^2\varphi_t\\
&\ge 0\,.
\end{align*}
Hence $r_t\ge |\beta+\alpha\cos\varphi_t|\ge \beta+\alpha\cos\varphi_t$, and thus
\[
r_t+\alpha+\beta\cos\varphi_t
\ge
(\alpha+\beta)(1+\cos\varphi_t)\,.
\]
It follows that
\[
\tan\frac{\varphi_{t+1}}{2}
\le
\frac{\beta}{\alpha+\beta}\cdot
\frac{\sin\varphi_t}{1+\cos\varphi_t}
=
\frac{\beta}{\alpha+\beta}
\tan\frac{\varphi_t}{2}\,.
\]
Iterating this inequality gives
\[
  \tan\frac{\varphi_t}{2}
  \le \left(\frac{\beta}{\alpha+\beta}\right)^t \tan\frac{\varphi_0}{2}\,.\qedhere
\]
\end{proof}

Regularization controls the scale of performative effects and then requires decoding. Normalization discards scale each round and contracts the half-angle directly at rate $\beta/(\alpha+\beta)$. Arbitrarily large positive feedback is therefore stable in direction, although convergence slows as $\beta/\alpha$ grows.

\subsection{Operator-valued and time-varying feedback}

The preceding theorem gives a global convergence result under constant performative feedback. We now allow the feedback to vary across directions and rounds. 

Let the stable signal $u\in\cH$ satisfy $\|u\|=1$, let $\alpha_t>0$, and consider
\[
z_{t+1} =\frac{\alpha_tu+B_tz_t}{\|\alpha_tu+B_tz_t\|},
\qquad \|z_t\|=1\,,
\]
where each $B_t:\cH\to\cH$ is bounded and linear. The feedback may vary across directions and rounds, while the assumptions below require $u$ to remain a common left and right eigenvector of $B_t$.

\begin{theorem}[Angular convergence with operator-valued and time-varying feedback]
\label{thm:operator-normalized-deployment}
Assume that, for every $t$,
\begin{equation}
\label{eq:operator-signal-direction}
B_tu=\beta_tu,
\qquad
B_t^*u=\beta_tu,
\qquad \beta_t\ge0\,,
\end{equation}
and define
\begin{equation*}
M_t=\sup_{\substack{v\perp u\\ \|v\|=1}}\|B_tv\|,
\qquad
q_t=\frac{M_t}{\alpha_t+\beta_t}\,.
\end{equation*}
If $\beta_0>0$, suppose
\begin{equation*}
\langle z_0,u\rangle> -\frac{\alpha_0}{\beta_0};
\end{equation*}
when $\beta_0=0$, no initialization condition is needed. Then $\langle z_1,u\rangle>0$, and for every $t\ge1$,
\begin{equation*}
\tan\varphi_{t+1}
\le q_t\tan\varphi_t\,,
\end{equation*}
where $\varphi_t$ is the angle between $z_t$ and $u$. Consequently,
\begin{equation}
\label{eq:operator-angle-product}
\tan\varphi_t
\le
\left(\prod_{k=1}^{t-1}q_k\right)\tan\varphi_1\,.
\end{equation}
If the product tends to zero, then $z_t\to u$. In particular, a uniform bound $q_t\le q<1$ gives geometric angular convergence. If $\langle z_0,u\rangle>0$, the bounds hold from $t=0$.
\end{theorem}

\begin{proof}
Write $z_t=c_tu+v_t,$ $c_t=\langle z_t,u\rangle,$ with $v_t\perp u\,.$
By \cref{eq:operator-signal-direction},
\[
\langle B_tv_t,u\rangle = \langle v_t,B_t^*u\rangle = \beta_t\langle v_t,u\rangle = 0,
\]
so $B_tv_t\perp u$. Therefore,
\[
\alpha_tu+B_tz_t = (\alpha_t+\beta_tc_t)u+B_tv_t\,.
\]
The initialization condition gives $\alpha_0+\beta_0c_0>0$, including when $\beta_0=0$, and hence $c_1>0$. If $c_t>0$, then $\alpha_t+\beta_tc_t>0$, so $c_{t+1}>0$. It follows that $c_t>0$ for every $t\ge1$.

For $t\ge1$, we have $c_t=\cos\varphi_t>0$ and
$\|v_t\| = \sin\varphi_t = c_t\tan\varphi_t\,.$
Therefore,
\begin{align*}
\tan\varphi_{t+1}
= \frac{\|B_tv_t\|}{\alpha_t+\beta_tc_t}
\le \frac{M_t\|v_t\|}{\alpha_t+\beta_tc_t}
= \frac{M_tc_t}{\alpha_t+\beta_tc_t}\tan\varphi_t
\le \frac{M_t}{\alpha_t+\beta_t}\tan\varphi_t
= q_t\tan\varphi_t\,,
\end{align*}
where we used $c_t\le1$. Iterating gives \cref{eq:operator-angle-product}.

If $\langle z_0,u\rangle>0$, the same argument applies from $t=0$.
\end{proof}

\begin{remark}[Random initialization]
If $z_0$ is drawn uniformly from the unit sphere in dimension $d$, then $\langle z_0,u\rangle$ concentrates near zero. Hence, for any fixed $\alpha_0/\beta_0>0$, the basin condition
\[
\langle z_0,u\rangle> -\frac{\alpha_0}{\beta_0}
\]
holds with probability $1-o(1)$ as $d\to\infty$.
\end{remark}

For a fixed operator, take $\alpha_t=\alpha$ and $B_t=B$. In the isotropic case $B_t=\beta_tI$, the same half-angle argument gives the contraction factor $\beta_t/(\alpha_t+\beta_t)$ for $\tan(\varphi_t/2)$. This bound holds globally and does not require the basin condition above.

\section{Data feedback loops in language modeling}
\label{sec:language-model-feedback}

We consider repeated language-model training when each new training set contains a source of high-quality real data, as well as synthetic data produced by earlier models. The high-quality source acts as a stable signal, while the synthetic data corresponds to the model's own influence on the training process. Such data feedback loops in language model training are an instance of performativity that is of independent interest. We'll see that the main message from earlier sections applies here, too: A relatively small amount of real data stabilizes training even in the presence of abundant synthetic data.

Let $C$ denote the context and $Y\in\{1,\ldots,V\}$ the next token. After round $t$, the model predicts a conditional distribution $P_t(\cdot\mid c)$. Let $Q_t(\cdot\mid c)$ denote the conditional distribution of the data used for the next retraining round. At the population level, cross entropy is minimized separately for each context, so retraining simply gives
\[
P_{t+1}=Q_t\,.
\]
The frequency of contexts may change across rounds without affecting this identity, provided every relevant context remains represented. 

\subsection{Time-varying mixture distributions}

Suppose the conditional next-token distribution at round $t+1$ is
\begin{equation}
\label{eq:next-token-mixture}
Q_t(\cdot\mid c)
=(1-\lambda_t)P^*(\cdot\mid c)
  +\lambda_tP_t(\cdot\mid c),
\qquad 0\le\lambda_t\le1\,,
\end{equation}
where $P^*$ is the stable conditional distribution. The quantity $1-\lambda_t$ is the fresh-data fraction at round $t$, while $\lambda_t$ is the synthetic fraction. This is the population version of the feedback loop studied in work on model collapse and model autophagy \citep{shumailov2024collapse,alemohammad2024mad}.

\begin{proposition}[Cross entropy under time-varying mixture feedback]
\label{prop:language-model-time-varying-feedback}
Assume population cross-entropy minimization at every round and the mixture response \cref{eq:next-token-mixture}. Then
\[
P_{t+1}=Q_t
\]
and, for every context $c$,
\begin{equation}
\label{eq:next-token-product}
P_t(\cdot\mid c)-P^*(\cdot\mid c)
=
\left(\prod_{s=0}^{t-1}\lambda_s\right)
\bigl(P_0(\cdot\mid c)-P^*(\cdot\mid c)\bigr)\,.
\end{equation}
Unless $P_0=P^*$ already, convergence to $P^*$ occurs if and only if
\[
\prod_{s=0}^{\infty}\lambda_s=0\,.
\]
Equivalently, this holds if some $\lambda_s=0$, or, when every $\lambda_s>0$, if and only if
\begin{equation}
\label{eq:cumulative-fresh-data}
\sum_{s=0}^{\infty}(1-\lambda_s)=\infty\,.
\end{equation}
In particular, a fixed fraction $\lambda_t\equiv\lambda<1$ gives the geometric rate $\lambda^t$.
\end{proposition}

\begin{proof}
Cross entropy is minimized pointwise by the conditional distribution of the training data, which is $Q_t$; changing the context marginal does not affect this population minimizer under common support. Subtracting $P^*$ from \cref{eq:next-token-mixture} gives
\[
P_{t+1}-P^*=\lambda_t(P_t-P^*)\,,
\]
and iteration gives \cref{eq:next-token-product}. The product vanishes immediately if some $\lambda_s=0$. Otherwise,
\[
\prod_s\lambda_s=0
\quad\Longleftrightarrow\quad
\sum_s -\log\lambda_s=\infty\,.
\]
Since $-\log\lambda\ge1-\lambda$, divergence of \cref{eq:cumulative-fresh-data} is sufficient. Conversely, if $\sum_s(1-\lambda_s)<\infty$, then $\lambda_s\to1$ and, eventually, $-\log\lambda_s\le2(1-\lambda_s)$, so the logarithmic series converges and the product is positive.
\end{proof}

A fixed fresh-data fraction is sufficient but not necessary: the fraction may shrink so long as its sum diverges. With pure self-consumption, $\lambda_t=1$, the population model does not move. 

\subsection{Accumulating data pools}

The mixture model above combines fresh data with samples of the latest model only. Practical pipelines rarely discard data: each round adds newly generated and newly collected data to a pool that retains everything from earlier rounds \citep{gerstgrasser2024model}. 

The pool contains data from all previous models. But after optimal retraining, the current model already matches the conditional distribution of the full pool, so the next update depends only on the current model and the new data. If real and synthetic data use the same context distribution, their proportions in the pooled dataset are the same at every context.

\begin{proposition}[Accumulated synthetic data]
\label{prop:accumulation}
Assume the initial pool and every real and synthetic batch use the same distribution over contexts.
Let the initial pool consist of $n_0>0$ units of data and let $P_0$ denote its conditional distribution, the model obtained by population cross-entropy minimization on this pool. At round $t=0,1,2,\ldots$, add $n_t^{\mathrm{gen}}\ge0$ units sampled from the deployed model $P_t$ and $n_t^{\mathrm{real}}\ge0$ fresh units with conditional distribution $P^*$, and let $P_{t+1}$ be the population cross-entropy minimizer on the enlarged pool of size $N_{t+1}=n_0+\sum_{s\le t}(n_s^{\mathrm{real}}+n_s^{\mathrm{gen}})$. Then, for every $t\ge0$,
\begin{equation}
\label{eq:accumulation-one-step}
P_{t+1}-P^*
=\Bigl(1-\frac{n_t^{\mathrm{real}}}{N_{t+1}}\Bigr)\bigl(P_t-P^*\bigr)\,,
\end{equation}
and therefore
\begin{equation}
\label{eq:accumulation-product}
P_t-P^*
=\prod_{u=0}^{t-1}\Bigl(1-\frac{n_u^{\mathrm{real}}}{N_{u+1}}\Bigr)
\bigl(P_0-P^*\bigr)\,.
\end{equation}
Unless $P_0=P^*$ already, convergence to $P^*$ occurs if and only if
\begin{equation}
\label{eq:accumulation-criterion}
\sum_{t=0}^{\infty}\frac{n_t^{\mathrm{real}}}{N_{t+1}}=\infty\,.
\end{equation}
\end{proposition}

\begin{proof}
Population cross-entropy minimization returns the conditional distribution of the training data, so every model equals the conditional distribution of its pool: $P_0$ by definition, and each $P_{t+1}$ by construction. Since the pool of size $N_t$ on which $P_t$ was trained has conditional distribution $P_t$, adding $n_t^{\mathrm{gen}}$ units with distribution $P_t$ leaves the conditional distribution unchanged, and adding $n_t^{\mathrm{real}}$ fresh units gives
\[
N_{t+1}P_{t+1}=(N_t+n_t^{\mathrm{gen}})P_t+n_t^{\mathrm{real}}P^*\,.
\]
Dividing by $N_{t+1}=N_t+n_t^{\mathrm{gen}}+n_t^{\mathrm{real}}$ and subtracting $P^*$ yields \cref{eq:accumulation-one-step}; iteration yields \cref{eq:accumulation-product}. The factors lie in $(0,1]$ because $0\le n_u^{\mathrm{real}}<N_{u+1}$. If the sum in \cref{eq:accumulation-criterion} diverges, then $\sum_u-\log(1-n_u^{\mathrm{real}}/N_{u+1})=\infty$ and the product vanishes. Conversely, if it converges, then $n_u^{\mathrm{real}}/N_{u+1}\to0$, eventually $-\log(1-n_u^{\mathrm{real}}/N_{u+1})\le2n_u^{\mathrm{real}}/N_{u+1}$, and the product is positive.
\end{proof}

\begin{remark}[Constant additions]
Suppose $n_u^{\mathrm{real}}=r>0$ and $n_u^{\mathrm{gen}}=g\ge0$ are constant. Then
\[
\frac{n_u^{\mathrm{real}}}{N_{u+1}}
=\frac{r}{n_0+(u+1)(r+g)}\,.
\]
Using $\log(1-x)=-x+O(x^2)$ in \cref{eq:accumulation-product} gives
\[
\log\prod_{u=0}^{t-1}
\left(1-\frac{r}{n_0+(u+1)(r+g)}\right)
=-\frac{r}{r+g}\log t+O(1)\,.
\]
Thus the product is $\Theta(t^{-r/(r+g)})$. Retaining all previously generated data therefore leads to polynomial rather than geometric convergence, with exponent equal to the fresh share of the per-round additions.
\end{remark}

\section{Parametric models and stochastic optimization}
\label{sec:parametric-models}

In the standard formulation of performative prediction~\citep{perdomo2020performative}, the distribution map~$D(\theta)$ depends on a $d$-dimensional parameter vector $\theta\in\mathbb{R}^d$ that specifies a parametric model such as $f_\theta(x)=\langle \theta,x\rangle.$ We'll now show that the stable signal principle extends to this setting. Throughout we assume that the distribution $D(\theta)$ has covariance
\[
\Sigma(\theta)=\E_{D(\theta)}[XX^\top]\,.
\]

\subsection{Exact population-level retraining}
\label{subsec:parametric-population}

The affine parameter update is not specific to squared loss. If the conditional risk is uniquely minimized at $\langle a^*+B\theta,x\rangle$ and $\Sigma(\theta)$ is positive definite, then population retraining gives
\[
T(\theta)=a^*+B\theta\,.
\]

For squared loss, this conditional-minimizer condition is equivalent to the conditional-mean model
\[
\E_{D(\theta)}[Y\mid X=x]=\langle a^*+B\theta,x\rangle\,.
\]
For logistic loss, it is enough that the conditional log-odds equal $\langle a^*+B\theta,x\rangle$.

Squared loss with $\ell_2$-regularization as in \Cref{sec:least-squares} gives 
\[
T_\delta(\theta)=\frac{a^*+B\theta}{1+\delta}\,.
\]
Indeed, the normal equations are
\[
(1+\delta)\Sigma(\theta)T_\delta(\theta) =\Sigma(\theta)(a^*+B\theta)\,,
\]
so the changing covariance cancels. If $(1+\delta)I-B$ is invertible, the regularized fixed point is
\begin{equation*}
\bar\theta_\delta=((1+\delta)I-B)^{-1}a^*\,,
\end{equation*}
and
\begin{equation*}
\theta_t-\bar\theta_\delta =\left(\frac{B}{1+\delta}\right)^t (\theta_0-\bar\theta_\delta)\,.
\end{equation*}
Thus retraining converges from every initialization if and only if $\rho(B)<1+\delta$. For the original unregularized loss,
\[
\theta_{\mathrm{PS}}=(I-B)^{-1}a^*\,,
\]
whenever $I-B$ is invertible, and
\[
 \theta_{\mathrm{PS}} = (I-B)^{-1}((1+\delta)I-B)\bar\theta_\delta\,.
\]
The increment in each round satisfies
\[
\theta_{t+2}-\theta_{t+1} =\frac{B}{1+\delta}(\theta_{t+1}-\theta_t)\,.
\]
A collection of linearly independent observed increments identifies $B$ and a subsequent transition identifies $a^*$.

\subsection{Stochastic gradient retraining for squared loss}

We now consider stochastic gradient retraining with one sample per deployment. We focus on squared loss. The population result only needs the pointwise risk minimizer to be well specified. A single stochastic gradient step also depends on the feature covariance because it does not solve the normal equations. We state the result for the scalar feedback model
\begin{equation}
\label{eq:parametric-scalar-conditional-mean}
\E_{D(\theta)}[Y\mid X=x] =\langle\alpha\theta^*+\beta\theta,x\rangle\,.
\end{equation}
At round $t$, draw $(X_t,Y_t)\sim D(\theta_t)$ and take one step
\begin{equation*}
\theta_{t+1} =\theta_t -\eta_t\Bigl((1+\delta)\langle\theta_t,X_t\rangle-Y_t\Bigr)X_t\,.
\end{equation*}
We assume that each update deploys the new model, corresponding to the ``Greedy Deploy'' method of \citet{mendler2020stochastic} in performative prediction.

Write
\[
\xi_{t+1} =Y_t-\langle\alpha\theta^*+\beta\theta_t,X_t\rangle\,.
\]
By \cref{eq:parametric-scalar-conditional-mean}, $\E[\xi_{t+1}\mid X_t,\theta_t]=0$.

Assume constants $\kappa,R,\nu>0$ such that, for every deployed parameter,
\begin{equation}
\label{eq:parametric-sgd-coverage}
\Sigma(\theta)\succeq\kappa I, \qquad
\|X\|_2\le R\quad\text{almost surely under }D(\theta)\,,
\end{equation}
and
\begin{equation}
\label{eq:parametric-sgd-noise}
\E[\xi_{t+1}^2\mid X_t,\theta_t]\le\nu^2\,.
\end{equation}

The next result follows from the standard second-moment analysis of stochastic gradient descent \citep[Section~4]{bottou2018optimization}.

\begin{proposition}[Stochastic recovery in the parametric model]
\label{prop:parametric-sgd}
Let $m=1+\delta-\beta$ and $\bar\theta=\frac{\alpha}{m}\theta^*$ and assume $m>0$. Under \cref{eq:parametric-scalar-conditional-mean,eq:parametric-sgd-coverage,eq:parametric-sgd-noise}, let $e_t=\theta_t-\bar\theta$. For every $t$ such that $\eta_t\le 2/(mR^2)$,
\begin{equation}
\label{eq:parametric-sgd-second-moment}
\E[\|e_{t+1}\|_2^2\mid\theta_t]
\le
\left[1-\kappa\bigl(2m\eta_t-m^2R^2\eta_t^2\bigr)\right]
\|e_t\|_2^2
+\eta_t^2\nu^2R^2\,.
\end{equation}
If $\sum_t\eta_t=\infty$, $\sum_t\eta_t^2<\infty$, and $\eta_t\le1/(mR^2)$ eventually, then
\[
\E[\|\theta_t-\bar\theta\|_2^2]\longrightarrow0\,.
\]
If $\beta\ne1$, the decoded iterates
\[
\widetilde\theta_t = \frac{1+\delta-\beta}{1-\beta}\theta_t
\]
converge in mean square to
\[
\theta_{\mathrm{PS}}=\frac{\alpha}{1-\beta}\theta^*\,.
\]
For a constant stepsize $0<\eta<2/(mR^2)$,
\begin{equation*}
\limsup_{t\to\infty}
\E[\|\theta_t-\bar\theta\|_2^2]
\le \frac{\eta\nu^2R^2}{\kappa(2m-m^2R^2\eta)}\,.
\end{equation*}
\end{proposition}

\begin{proof}
Using $m\bar\theta=\alpha\theta^*$, the update becomes
\[
e_{t+1} =e_t-\eta_t \bigl(m\langle e_t,X_t\rangle-\xi_{t+1}\bigr)X_t\,.
\]
Conditioning on $\theta_t$ and using $\E[\xi_{t+1}\mid X_t,\theta_t]=0$ gives
\begin{align*}
\E[\|e_{t+1}\|_2^2\mid\theta_t]
&= \|e_t\|_2^2
  - 2m\eta_t e_t^\top\Sigma(\theta_t)e_t\\
& \quad
  + m^2\eta_t^2
   \E[\|X_t\|_2^2\langle e_t,X_t\rangle^2\mid\theta_t]
  + \eta_t^2\E[\xi_{t+1}^2\|X_t\|_2^2\mid\theta_t]\,.
\end{align*}
The assumptions imply
\[
\E[\|X_t\|_2^2\langle e_t,X_t\rangle^2\mid\theta_t]
\le R^2 e_t^\top\Sigma(\theta_t)e_t
\]
and
\[
\E[\xi_{t+1}^2\|X_t\|_2^2\mid\theta_t]\le\nu^2R^2\,.
\]
When $\eta_t\le 2/(mR^2)$, the coefficient $2m\eta_t-m^2R^2\eta_t^2$ is nonnegative. Substitution and $\Sigma(\theta_t)\succeq\kappa I$ then give \cref{eq:parametric-sgd-second-moment}. The decreasing- and constant-stepsize conclusions follow by iterating this recursion.
\end{proof}
The conditional mean update is
\begin{equation*}
\E[\theta_{t+1}-\theta_t\mid\theta_t]
=-\eta_t m\Sigma(\theta_t)(\theta_t-\bar\theta)\,.
\end{equation*}
Thus population retraining needs only uniqueness of the minimizing parameter, while this SGD argument requires a uniform lower eigenvalue bound.

\section{Conclusion}

The stable signal perspective is an organizing principle for understanding repeated model training and deployment. It predicts that retraining seeks stable signals in the data, even if the initial influence of the stable signal is weak compared to the model’s own influence. If a click on a platform is 90\% the model’s display ranking, and 10\% item quality, repeated optimization will---perhaps counterintuitively---settle on item quality eventually. Previous theories of retraining, in contrast, apply primarily to small performative effects, casting model influence as an error term rather than a first-class citizen.

With suitable regularization or normalization, stable signals can overcome arbitrarily strong performative effects. This result holds with some generality: under arbitrary feature changes, time-varying and heterogeneous performative effects, and operator-valued extensions. What retraining converges to, however, isn’t the stable signal itself, but rather the direction of the stable signal. This suggests that the natural stabilizing behavior of retraining is directional, not absolute. What survives aren’t absolute numbers, but rather the ranking of alternatives according to the stable signal.

\subsection*{Acknowledgments}

This is a companion article to an invited contribution to the Proceedings of the International Congress of Mathematicians (ICM), 2026. Many thanks to the organizers for this opportunity. I'd like to thank Jiduan Wu for her contributions to a closely related ongoing project. Thanks to Juan Carlos Perdomo, Celestine Mendler-Dünner, and Lorenzo Rosasco for helpful feedback.

\bibliographystyle{plainnat}
\bibliography{references}

\end{document}